\documentclass{article} 
\usepackage{fullpage}
\usepackage{amsmath,amssymb,amsthm,amsfonts,latexsym,xspace,enumerate,graphicx,color,eucal,upgreek,mathtools}
\usepackage[backref,colorlinks,bookmarks=true]{hyperref}
\usepackage{nicefrac}
\usepackage{float}

\makeatletter
\floatstyle{ruled}
\newfloat{fragment}{H}{lop}
\floatname{fragment}{Algorithm}
\renewcommand{\floatc@ruled}[2]{\vspace{2pt}{\@fs@cfont #1.\:} #2 \par
 \vspace{1pt}}
\makeatother

\newtheorem{theorem}{Theorem}[section]
\newtheorem*{namedtheorem}{\theoremname}
\newcommand{\theoremname}{testing}

\newtheorem{thm}[theorem]{Theorem}
\newtheorem{lemma}[theorem]{Lemma}

\newtheorem{claim}[theorem]{Claim}

\newtheorem{fact}[theorem]{Fact}

\theoremstyle{definition}
\newtheorem{definition}[theorem]{Definition}

\theoremstyle{plain}
\newtheorem{Alg}{Algorithm}

\renewenvironment{proof}{\noindent{\textbf{Proof:}}} {$\blacksquare$\vskip \belowdisplayskip}

\newcommand{\ignore}[1]{}

\newcommand{\E}{\mathop{\bf E\/}}

\newcommand{\Proj}{\mbox{Proj}}
\newcommand{\Sp}{\mathbb{S}}

\newcommand{\poly}{\mathrm{poly}}

\newcommand{\R}{\mathbb R}

\newcommand{\calH}{\mathcal{H}}

\newcommand{\norm}[1]{\left\lVert #1 \right\rVert}

\newcommand{\set}[1]{\left\{ #1 \right\}}

\title{Provable ICA with Unknown Gaussian Noise, and Implications for Gaussian Mixtures and Autoencoders}

\author{
Sanjeev Arora \thanks{Department of Computer Science, Princeton University, \texttt{arora@cs.princeton.edu}}
\and
Rong Ge \thanks{
Department of Computer Science, Princeton University, \texttt{rongge@cs.princeton.edu}}
\and
Ankur Moitra \thanks{School of Mathematics, Institute for Advanced Study, \texttt{moitra@ias.edu}}
\and
Sushant Sachdeva \thanks{Department of Computer Science, Princeton University, \texttt{sachdeva@cs.princeton.edu}}
}

\begin{document}

\maketitle

\begin{abstract}
We present a new algorithm 
for Independent Component Analysis (ICA) which has provable performance guarantees.
In particular, suppose we are given samples of the form
$y = Ax + \eta$ where $A$ is an unknown $n \times n$ matrix and $x$ is a random variable
whose components are independent and have a fourth moment strictly less than that of a standard Gaussian
random variable and $\eta$ is an
$n$-dimensional Gaussian random variable with unknown
covariance $\Sigma$: 
We give an algorithm that provable recovers $A$ and $\Sigma$ up to
an additive $\epsilon$ and whose running time and sample complexity are polynomial in $n$ and $1 / \epsilon$. 
To accomplish this, we introduce a novel ``quasi-whitening'' step that may
be useful in other contexts in which the covariance of Gaussian noise is not known in
advance. We also give a general framework for finding all local optima of a function
(given an oracle for approximately finding just one) and this is a crucial step in our
algorithm, one that has been overlooked in previous attempts, and allows us to
control the accumulation of error when we find the columns of $A$ one by one via local search.
\end{abstract}

\section{Introduction}

We present an algorithm (with rigorous performance
guarantees) for a basic statistical problem. Suppose $\eta$ is an
independent $n$-dimensional Gaussian random variable with an unknown
covariance matrix $\Sigma$ and  $A$ is an unknown $n \times n$ matrix. We are given samples of the form
$y = Ax + \eta$ where $x$ is a random variable
whose components are independent and have a fourth moment strictly less than that of a standard Gaussian
random variable. The most natural case is when $x$ is chosen
uniformly at random from $\{+1, -1\}^n$, although our algorithms in even the more general case above.  
Our goal is to reconstruct an additive approximation to the
matrix $A$ and the covariance matrix $\Sigma$ running in time and using a number of samples that is polynomial in $n$ and $\frac{1}{\epsilon}$, where $\epsilon$ is the target precision (see Theorem~\ref{thm:main})
This problem arises in several research directions within machine
learning: Independent Component Analysis (ICA), Deep Learning, Gaussian Mixture Models (GMM), etc. We describe
these connections next, and known results (focusing on algorithms with provable performance guarantees, since that is our goal). 

Most obviously, the above problem can be seen as an instance of {\em Independent Component Analysis} (ICA) with
unknown Gaussian noise.  ICA has an illustrious history with applications ranging from econometrics, to signal processing, to image segmentation. 
The goal generally involves finding a linear transformation of the data so that the
coordinates are as independent as possible~\cite{C, HKO, HO}.  This is often accomplished by
finding directions in which the projection is ``non-Gaussian"
\cite{Hub}. 
 Clearly, if the datapoint $y$ is generated as $Ax$ (i.e., with no
 noise $\eta$ added)  then applying linear transformation $A^{-1}$ to the
 data results in samples $A^{-1} y$ whose coordinates are independent. 
This restricted case was considered by 
Comon \cite{C} and Frieze, Jerrum and Kannan \cite{FJK}, and their goal was to recover an additive
approximation to $A$ efficiently and using a polynomial number of samples. (We
will later note  a gap in their reasoning, albeit  fixable
by our methods. See also recent papers by Anandkumar {\em et al.}, Hsu and Kakade\cite{AFHKL, HK}, that do not use local search and avoids this issue.) To the best of our knowledge, there are currently no known algorithms
with provable guarantees for the more general case of ICA with
Gaussian noise (this is especially true if the covariance matrix is unknown, as in our problem), although
many empirical approaches are known.  (eg. \cite{LCC}, the issue of ``empirical'' 
vs ``rigorous'' is elaborated upon after Theorem~\ref{thm:main}.)

The second view of our problem is as a concisely described
{\em Gaussian Mixture Model}. Our data is generated as a mixture of $2^n$ identical
Gaussian components (with an unknown covariance matrix)  whose centers are the
points $\set{Ax: x \in \set{-1,1}^n}$, and all mixing weights are equal.
Notice, this  mixture of $2^n$ Gaussians admits a concise description using $O(n^2)$ parameters.
The problem of learning Gaussian mixtures has a long history, and the
popular approach in practice is to use the EM algorithm~\cite{DLR}, 
though it has no worst-case guarantees (the method may take a very long time to
converge, and worse, may not always converge to the correct solution).
An influential paper of Dasgupta \cite{Das} initiated the program of
designing algorithms with provable guarantees, which was improved
in a sequence of papers~\cite{AK,BS2,KMV,MV}.
But in the current setting, it is unclear how to apply any of the
 above algorithms (including $EM$) since 
the trivial application would keep track of exponentially many
parameters -- one for each component. 
Thus, new ideas seem necessary to achieve polynomial running time.

The third view of our problem is as a simple form of  {\em
  autoencoding} \cite{HS}. This is a central notion in Deep Learning, where the goal is to obtain a compact
representation of a target distribution using a multilayered architecture,
where a complicated function (the target) can be built up by composing
layers of a simple function (called the autoencoder~\cite{Ben}). The main
tenet is that there are interesting functions which can be represented concisely using many layers, but would
need a very large representation if a ``shallow" architecture is used
instead). This is most useful for functions that are ``highly varying'' (i.e. cannot be
compactly described by piecewise linear functions or other ``simple"
local representations). Formally, it is possible to represent using just
(say) $n^2$
parameters, some distributions with $2^n$ ``varying parts'' or ``interesting 
regions.''    The {\em Restricted Boltzmann Machine} (RBM) is an
especially popular autoencoder in Deep Learning, though many others have been
proposed. However, to the best of our knowledge, there has been no
successful attempt to give a {\em rigorous} analysis of Deep Learning. 
Concretely, if the data is indeed generated using the distribution
represented by an RBM, then do the popular algorithms for Deep
Learning~\cite{hintontutorial}  learn the model parameters {\em correctly} and in
{\em polynomial} time? Clearly, if the running time were actually
found to be exponential in  the number of parameters, then this would erode some
of the advantages of the compact representation. 

How is Deep Learning related to our problem? As noted by Freund and Haussler~\cite{FS} many years ago,
an RBM with real-valued visible units (the version that seems more
amenable to theoretical analysis) is precisely a mixture of exponentially many standard
Gaussians. It is parametrized by an $n \times m$ matrix $A$ and a vector
$\theta \in \R^n$. It 
encodes a mixture of $n$-dimensional standard Gaussians centered at
the points  $\set{Ax: x \in \set{-1,1}^m}$, where the mixing weight
of the Gaussian centered at $Ax$ is $\exp(\|Ax\|_2^2 +\theta\cdot x)$.  This is of course
reminiscent of our problem. Formally, our
algorithm can be seen as a nonlinear autoencoding scheme analogous to an
RBM but with uniform mixing weights.
Interestingly, the algorithm that we present here looks nothing like the approaches favored
traditionally in Deep Learning, and may provide an interesting new perspective. 

\subsection{Our results and techniques}

We give a provable algorithm for ICA with unknown Gaussian noise. We have not made an attempt to optimize the quoted running time of this model, but we emphasize that this is in fact the first algorithm with provable guarantees for this problem and moreover we believe that in practice
our algorithm will run almost as fast as the usual ICA algorithms, which are its close relatives. 

\begin{thm}[Main, Informally] \label{thm:main} There is an algorithm that recovers the unknown $A$
  and $\Sigma$ up to additive error $\epsilon$ in each entry in time
  that is polynomial in $n, \|A\|_2, \|\Sigma\|_2, \nicefrac{1}{\epsilon},
  \nicefrac{1}{\lambda_{\min}(A)}$ where $\|\cdot\|_2$ denotes the operator norm and
  $\lambda_{\min}(\cdot)$ denotes the smallest eigenvalue.
\end{thm}

The classical approach for ICA initiated
in Comon \cite{C} and Frieze,
Jerrum and Kannan \cite{FJK}) is for the noiseless case in which $y=Ax$.  The
first step is {\em whitening}, which
applies a suitable linear transformation that makes the 
variance the same in all directions, thus reducing to 
the case where $A$ is a {\em rotation}
matrix. Given samples $y = Rx$ where $R$ is a rotation matrix, 
the rows of $R$ can be found in principle by computing 
the vectors $u$ that are local minima of $E[(u\cdot y)^4]$.
Subsequently, a number of works (see e.g. \cite{CCC, DDV}) have focused on giving algorithms
that are robust to noise. A popular approach is to use the fourth order {\em cumulant} 
(as an alternative to the fourth order moment) as a method for ``denoising,''
or any one of a number of other functionals whose local optima reveal interesting directions. However,
theoretical guarantees of these algorithms are not well understood.

The above procedures in the noise-free model can {\em almost} be made
rigorous (i.e., provably polynomial running time and number of
samples), except for one subtlety: it is unclear how to use local
search to find {\em all} optima in polynomial time. In practice, one
finds a single local optimum, projects to the subspace orthogonal to
it and continues recursively on a lower-dimensional problem. However,
a naive implementation of this idea is unstable since approximation
errors can accumulate badly, and to the best of our knowledge no
rigorous analysis has been given prior to our work.  (This is not a
technicality: in some similar settings the errors are known to blow up
exponentially \cite{VX}.)  One of our contributions is a modified
local search that avoids this potential instability and finds all
local optima in this setting. (Section~\ref{subsec:allopt}.)

Our major new contribution however is dealing
with noise that is an unknown Gaussian. This is an important
generalization, since many methods used in ICA are quite unstable
to noise (and a wrong estimate for the covariance could
lead to bad results). Here, we no longer need to assume we know
even rough estimates for the covariance. Moreover, in the context
of Gaussian Mixture Models this generalization corresponds to learning
a mixture of many Gaussians where the covariance of the components
is not known in advance. 

We design new tools for denoising and especially whitening in this setting. Denoising uses the
fourth order cumulant instead of the fourth moment used in \cite{FJK} and  whitening involves a novel use of the
Hessian of the cumulant. Even then, we cannot reduce to the simple
case $y=Rx$ as above, and are left with a more complicated functional
form (see  ``quasi-whitening'' in Section~\ref{sec:whitening}.) Nevertheless, we can reduce to an
optimization problem that can be solved via local search, and which
remains amenable to a rigorous analysis.
The results of the local optimization step can be
then used to simplify the complicated functional form and recover $A$
as well as the noise $\Sigma$. We defer many of our proofs to the supplementary material section, due to space constraints. 

In order to avoid cluttered notation, we have focused on the case in which $x$ is chosen uniformly at random from $\{-1, +1\}^n$, although our algorithm and analysis work under the more general conditions that the coordinates of $x$ are (i) independent and (ii) have a fourth moment that is less than three (the fourth moment of a Gaussian random variable). In this case, the functional $P(u)$ (see Lemma~\ref{lem:denoise}) will take the same form but with weights depending on the exact value of the fourth moment for each coordinate. Since we already carry through an unknown diagonal matrix $D$ throughout our analysis, this generalization only changes the entries on the diagonal and the same algorithm and proof apply.

\section{Denoising and quasi-whitening }
\label{sec:whitening}

As mentioned, our  approach  is based on the fourth order cumulant.
The cumulants of a random variable are the coefficients of
the Taylor expansion of the logarithm of the characteristic function \cite{KS}. 
Let $\kappa_r(X)$ be the $r^{th}$ cumulant of a random variable
$X$. We make use of:

\begin{fact}
(i) If $X$ has mean zero, then $\kappa_4(X) = \E[X^4] - 3 \E[X^2]^2$. 
(ii) If $X$ is Gaussian  with mean $\mu$ and variance $\sigma^2$,
then $\kappa_1(X) = \mu$, $\kappa_2(X) = \sigma^2$ and $\kappa_r(X) =
0$ for all $r > 2$. 
(iii) If $X$ and $Y$ are independent, then $\kappa_r(X+Y) = \kappa_r(X) + \kappa_r(Y)$. 
\end{fact}

The crux of our technique is to look at the following functional, where
$y$ is the random variable $Ax + \eta$ whose samples are given to
us. Let $u\in \R^n$ be any vector. Then
$P(u) = - \kappa_4(u^T y)$.
Note that for any $u$ we can compute $P(u)$ reasonably accurately by drawing
sufficient number of samples of $y$ and taking an empirical average.  
Furthermore, since $x$ and  $\eta$ are  independent, and $\eta$ is Gaussian, the next lemma is
immediate. We call it ``denoising'' since it allows us empirical access to some
information about $A$  that is uncorrupted by the noise $\eta$.

\begin{lemma}[Denoising Lemma] \label{lem:denoise}
$P(u) = 2 \sum_{i=1}^n (u^TA)_i^4$.
\end{lemma}


\begin{proof}
The crucial observation is that $u^Ty = u^TAx + u^T\eta$ is the sum of two independent random variables, $Ax$ and $\eta$ and that $P(u) =- \kappa_4(u^TAx + u^T\eta) =- \kappa_4(u^TAx) - \kappa_4(u^T\eta) = - \kappa_4(u^TAx)$. So in fact, the functional $P(u)$ is invariant under additive Gaussian noise \textbf{independent of the variance matrix} $\Sigma$. This vastly simplifies our computation:
 \begin{align*}
 \E[(u^TAx)^4] & = \sum_{i=1}^n (u^TA)_i^4  \E[x_i^4] + 6 \sum_{i < j } (u^TA)_i^2(u^TA)_j^2 \E[x_i^2]\E[x_j^2] \\
 & = \sum_{i=1}^n (u^TA)_i^4  + 6 \sum_{i < j } (u^TA)_i^2(u^TA)_j^2  = -2 \sum_{i=1}^n (u^TA)_i^4  + 3(u^TAA^Tu)^2
  \end{align*}
  Furthermore $\E[(u^T A x)^2]^2 = (u^TAA^Tu)^2$ and we conclude that $$P(u) = -\kappa_4(u^Ty) = - \E[(u^TAx)^4] + 3 \E[(u^T A x)^2]^2 = 2 \sum_{i=1}^n (u^TA)_i^4.$$
\end{proof}

\subsection{Quasi-whitening via the Hessian of $P(u)$}

In prior works on ICA, {\em whitening} refers to reducing to the case
where $y =Rx$ for some some rotation matrix $R$. Here we give a
technique to reduce to the case where $y = RDx + \eta'$ where $\eta'$
is some other Gaussian noise (still unknown), $R$ is a rotation matrix
and $D$ is a diagonal matrix that depends upon $A$. We call this {\em
  quasi-whitening.} Quasi-whitening suffices for us since local search
using the objective function $\kappa_4(u^T y)$ will give us
(approximations to) the
rows  of $RD$, from which we will be able to recover
$A$.

Quasi-whitening involves computing the Hessian of $P(u)$, which recall is the
matrix of all 2nd order partial derivatives of $P(u)$. Throughout this section, we will denote the Hessian operator by $\calH$. In matrix form, the Hessian of $P(u)$ is
$$ \frac{\partial^2}{\partial u_i \partial u_j} P(u) = 24 \sum_{k=1}^n A_{i, k} A_{j, k} (A_k\cdot u)^2\mbox{; } \calH(P(U)) = 24 \sum_{k=1}^n (A_k \cdot u)^2 A_k A_k^T = AD_A(u)A^T$$
where $A_k$ is the $k$-th column of the matrix $A$ (we use subscripts to denote the columns of matrices throught the paper). $D_A(u)$ is the following diagonal matrix:
\begin{definition} \label{defn:diagmatrix}
Let $D_A(u)$ be a diagonal matrix in which the $k^{th}$ entry is $24 (A_k \cdot u)^2$. 
\end{definition}

Of course, the exact Hessian of $P(u)$ is unavailable and we will instead compute an empirical approximation $\widehat{P}(u)$ to $P(u)$ (given many samples from the distribution), and we will show that the Hessian of $\widehat{P}(u)$ is a good approximation to the Hessian of $P(u)$.  
\begin{definition}
Given $2N$ samples $y_1, y'_1,  y_2, y'_2 ... , y_N, y'_N$ of the random variable $y$, let $$\widehat{P}(u) =  \frac{-1}{N} \sum_{i=1}^N (u^T y_i)^4 +  \frac{3}{N} \sum_{i=1}^N (u^T y_i)^2 (u^T y'_i)^2.$$
\end{definition}
Our first step is to show that the expectation of the Hessian of $\widehat{P}(u)$ is exactly the Hessian of $P(u)$. In fact, since the expectation of $\widehat{P}(u)$ is exactly $P(u)$ (and since $\widehat{P}(u)$ is an analytic function of the  samples and of the vector $u$), we can interchange the Hessian operator and the expectation operator. Roughly, one can imagine the expectation operator as an integral over the possible values of the random samples, and as is well-known in analysis, one can differentiate under the integral provided that all functions are suitably smooth over the domain of integration. 

\begin{claim}
$\E_{y, y'}[- (u^T y)^4 + 3 (u^T y)^2 (u^T y')^2] = P(u)$
\end{claim}

This claim follows immediately from the definition of $P(u)$, and since $y$ and $y'$ are independent. 

\begin{lemma}
$\calH(P(u)) = \E_{y, y'}[ \calH(- (u^T y)^4 + 3 (u^T y)^2 (u^T y')^2)]$
\end{lemma}


Next, we compute the two terms inside the expectation:

\begin{claim}
$\calH((u^T y)^4 ) = 12 (u^T y)^2 y y^T $
\end{claim}

\begin{claim}
$\calH( (u^T y)^2 (u^T y')^2) =2 (u^T y')^2 y y^T + 2 (u^T y)^2 y' (y')^T + 4 (u^Ty)(u^T y') (y (y')^T + (y')y^T)$
\end{claim}


Let $\lambda_{min}(A)$ denote the smallest eigenvalue of $A$. Our
analysis also requires bounds on the entries of $D_A(u_0)$: 

\begin{claim}\label{claim:obv}
If $u_0$ is chosen uniformly at random then with high probability for all $i$, $$\min_{i=1}^n \|A_i\|_2^2 n^{-4} \le D_A(u_0)_{i,i} \le \max_{i=1}^n \|A_i\|_2^2 \frac{\log n}{n}$$
\end{claim}

\begin{proof}
We can bound $\max_{i=1}^n |A_i\cdot u|$ by $\max_{i=1}^n \|A_i\|_2 \frac{\log n}{\sqrt{n}}$ thus the bound for $\max_{i=1}^n(D_A(u_0))_{i,i}$ follows. Note that with high probability the minimum absolute value of $n$ Gaussian random variables is at least $1/n^2$, hence $\min_{i=1}^n (D_A(u_0))_{i,i} \ge \min_{i=1}^n \|A_i\|_2^2 n^{-4}$.
\end{proof}

\begin{lemma} If $u_0$ is chosen uniformly at random and furthermore we are given $2N =  \poly(n, 1/\epsilon, 1/\lambda_{min}(A),$ $ \|A\|_2, \|\Sigma\|_2)$ samples of $y$, then with high probability we will have that $(1-\epsilon) AD_A(u_0)A^T \preceq \mathcal{H}(\widehat{P}(u_0)) \preceq (1+\epsilon) AD_A(u_0)A^T$. 
\end{lemma}

\begin{proof}
First we consider each entry of the matrix updates. For example, the variance of any entry in $\calH((u^Ty)^4) = 12 (u^T y)^2 y y^T$ can be bounded by $\|y\|_2^8$, which we can bound by $\E[\|y\|_2^8] \le O(\E[\|A x\|_2^8+ \|\eta\|_2^8])$. This can be bounded by $O(n^4 (\|A\|_2^8+\|\Sigma\|_2^4))$. This is also an upper bound for the variance (of any entry) of any of the other matrix updates when computing $\calH(\widehat{P}(u_0))$. 

Applying standard concentration bounds, $\poly(n, 1/\epsilon', \|A\|_2, \|\Sigma\|_2)$ samples suffice to guarantee that all entries of $\mathcal{H}(\widehat{P}(u_0))$ are $\epsilon'$ close to $\calH(P(u))$. The smallest eigenvalue of $\calH(P(u)) = AD_A(u_0)A^T$ is at least $\lambda_{min}(A)^2 \min_{i=1}^n \|A_i\|_2^2 n^{-4}$ where here we have used Claim~\ref{claim:obv}. If we choose $\epsilon' = \poly(1/n, \lambda_{min}(A), \epsilon)$, then we are also guaranteed $(1-\epsilon) AD_A(u_0)A^T \preceq \mathcal{H}(\widehat{P}(u_0)) \preceq (1+\epsilon) AD_A(u_0)A^T$ holds.
\end{proof}

\begin{lemma}\label{lemma:rot}
Suppose that $(1-\epsilon) AD_A(u_0)A^T \preceq \widehat{M} \preceq (1+\epsilon) AD_A(u_0)A^T$, and let $\widehat{M}  = BB^T$. Then there is a rotation matrix $R^*$ such that $\|B^{-1} A D_A(u_0)^{1/2} - R^*\|_F \leq \sqrt{n} \epsilon$.
\end{lemma}

The intuition is: if any of the singular values of $B^{-1} A D_A(u_0)^{1/2}$ are outside the range  $[1 - \epsilon, 1 + \epsilon]$, we can find a unit vector $x$ where the quadratic forms $x^T AD_A(u_0)A^T x $ and $x^T \widehat{M}x$ are too far apart (which contradicts the condition of the lemma). Hence the singular values of $B^{-1} A D_A(u_0)^{1/2}$ can all be set to one without changing the Froebenius norm of $B^{-1}A D_A(u_0)^{1/2}$ too much, and this yields a rotation matrix. 

\begin{proof}
Let $M = AD_A(u_0)A^T$ and let $C = A D_A(u_0)^{1/2}$, and so $M = C C^T$ and $\widehat{M} = B B^T$. The condition $(1-\epsilon) M \preceq \widehat{M} \preceq (1+\epsilon) M$ is well-known to be equivalent to the condition that for all vectors $x$, $(1-\epsilon)x^T M x \leq x^T \widehat{M} x \leq (1+\epsilon) x^TMx$. 

Suppose for the sake of contradiction that $S = B^{-1}C$ has a singular value outside the range $[1 - \epsilon, 1 + \epsilon]$. Assume (without loss of generality) that $S$ has a singular value strictly larger than $1 + \epsilon$ (and the complementary case can be handled analogously). Hence there is a unit vector $y$ such that $y^T S S^T y > 1 + \epsilon$. But since $B S S^T B^T = C C^T$, if we set $x^T = y^T B^{-1}$ then we have $x^T\widehat{M}x = x^T B B^T x = y^T y = 1$ but $x^TMx = x^T C C^T x = x^T B S S^T B^T x = y^T S S^T y > 1 + \epsilon$.  This is a contradiction and so we conclude that all of the singular values of $B^{-1}C$ are in the range $[1 - \epsilon, 1 + \epsilon]$.

Let $U \Sigma V^T$ be the singular value decomposition of $B^{-1}C$. If we set all of the diagonal entries in $\Sigma$ to $1$ we obtain a rotation matrix $R^* = U V^T$. And since the singular values of $B^{-1}C$ are all in the range $[1 - \epsilon, 1 + \epsilon]$, we can bound the Froebenius norm of $B^{-1}C - R^*$: $\|B^{-1}C - R^*\|_F \leq \sqrt{n} \epsilon$, as desired. 
\end{proof}

\section{Our algorithm (and notation)}

In this section we describe our overall algorithm. It uses as a
blackbox the denoising and quasi-whitening already described above,
as well as a routine for computing all local maxima
of some ``well-behaved'' functions  which is described later in
Section~\ref{sec:localopt}.

\noindent{\bf Notation:} Placing a hat over a function corresponds to an empirical
approximation that we obtain from random samples. This approximation
introduces error, which we will keep track of.

\noindent {\bf Step 1}: {\em Pick a random $u_0 \in \R^n$ and estimate the Hessian
  $\calH(\widehat{P}(u_0))$. Compute $B$ such that $\calH(\widehat{P}(u_0)) =
  BB^T$. Let $D= D_A(u_0)$ be the diagonal matrix defined in
  Definition~\ref{defn:diagmatrix}.}
 
\noindent{\bf Step 2:} {\em Take $2N$ samples $y_1, y_2, ..., y_N, y_1', y_2', ..., y_N'$, and let $$\widehat{P}'(u) = - \frac{1}{N} \sum_{i=1}^N (u^T B^{-1} y_i)^4  + \frac{3}{N} \left(\sum_{i=1}^N (u^T B^{-1} y_i)^2(u^TB^{-1}y_i')^2\right)$$ which is an empirical estimation of $P'(u)$.}

\noindent{\bf Step 3:} {\em Use the procedure {\sc AllOPT($\widehat{P}'(u), \beta, \delta', \beta', \delta'$)} of Section~\ref{sec:localopt} to
  compute all $n$ local maxima of the function $\widehat{P}'(u)$. }


\noindent{\bf Step 4:}  {\em Let $R$ be the matrix whose rows are the $n$
local optima recovered in the previous step. Use procedure {\sc
  Recover} 
of Section~\ref{sec:cumulant} to find $A$ and $\Sigma$.}

\noindent{\bf Explanation:} Step 1 uses the transformation $B^{-1}$ computed in the previous
Section to quasi-whiten the
data. Namely, we consider the sequence of samples $z= B^{-1}y$, which
are therefore of the form  $R'D x +\eta'$ where $\eta =B^{-1}\eta$, $D
= D_A(u_0)$ and $R'$ is
close to a rotation matrix $R^*$  (by
Lemma~\ref{lemma:rot}).
In Step 2 we look at $\kappa_4(
(u^T z))$, which effectively denoises the new samples (see Lemma~\ref{lem:denoise}), 
and thus is the same as $\kappa_4(R' D^{-1/2}x)$.
Let $P'(u) = \kappa_4 (u^Tz) = \kappa_4(u^TB^{-1}y)$ which is easily seen to be $E[(u^T R'
D^{-1/2}x)^4]$.  Step 2 estimates this function, obtaining $\widehat{P}'(u)$. 
Then Step 3 tries to find local optima via local search. 
Ideally we would have liked access to the functional $P^*(u) = (u^T R^*
x)^4$ since the procedure for local optima works only for true
rotations. But since $R'$ and $R^*$ are close we can make it work
approximately with
$\widehat{P}'(u)$, and then in Step 4 use these local optima to finally recover $A$.
  


\begin{theorem}
Suppose we are given samples of the form $y = Ax + \eta$ where $x$ is uniform on $\{+1, -1\}^n$, $A$ is an $n\times n$ matrix, $\eta$ is an $n$-dimensional Gaussian random variable independent of $x$ with unknown covariance matrix $\Sigma$. There is an algorithm that with high probability recovers $\|\widehat{A} - A\Pi\mbox{diag}(k_i)\|_F \le \epsilon$ where $\Pi$ is some permutation matrix and each $k_i\in\{+1, -1\}$ and also recovers $\|\widehat{\Sigma} - \Sigma\|_F \le \epsilon$. Furthermore the running time and number of samples needed are $\poly(n, 1/\epsilon, \norm{A}_2, \norm{\Sigma}_2, 1/\lambda_{min}(A))$ 
\end{theorem}

\begin{proof}
In Step 1, by Lemma~\ref{lemma:rot} we know once we use $z = B^{-1} y$, the whitened function $P'(u)$ is inverse polynomially close to $P^*(u)$. Then by Lemma~\ref{lem:functionvalue}, the function $\widehat{P'}(u)$ we get in Step 2 is inverse polynomially close to $P'(u)$ and $P^*(u)$.  Theorem~\ref{thm:allopt} and Lemma~\ref{lem:alllocalmax} show that given $\widehat{P'}(u)$ inverse polynomially close to $P^*(u)$, Algorithm~\ref{alg:allopt}: {\sc: AllOPT} finds all local maxima with inverse polynomial precision. Finally by Theorem~\ref{thm:recover} we know $A$ and $W$ are recovered correctly up to additive $\epsilon$ error in Frobenius norm. The running time and sampling complexity of the algorithm is polynomial because all parameters in these Lemmas are polynomially related.
\end{proof}

Note that here we recover $A$ up to a permutation of the  columns and sign-flips. In general, this is all we can hope for since the distribution of $x$ is also invariant under these same operations. Also, the dependence of our algorithm on the various norms (of $A$ and $\Sigma$) seems inherent since our goal is to recover an additive approximation, and as we scale up $A$ and/or $\Sigma$, this goal becomes a stronger relative guarantee on the error.

\section{Framework for iteratively finding all local maxima}
\label{sec:localopt}

In this section, we first describe a fairly standard procedure (based upon
Newton's method) for finding a
{\em single} local maximum of a function $f^*: \R^n \to \R$ among all unit vectors
and an analysis of its rate of convergence. Such a procedure is a common tool in
statistical algorithms, but here we state it rather carefully since we later give a
general method to convert any local search algorithm (that meets certain criteria)
into one that finds {\em all} local maxima (see Section~\ref{subsec:allopt}). 

Given that we can only ever hope for an additive approximation to a local maximum, one should be concerned
about how the error accumulates when our goal is to find {\em all} local maxima. 
In fact, a naive strategy is to project onto the subspace orthogonal to the directions
found so far, and continue in this subspace. However, such an approach seems to accumulate errors badly (the additive error of the last
local maxima found is exponentially larger than the error of the first). Rather, the crux of our analysis
is a novel method for bounding how much the error can accumulate (by
refining old estimates). 

Our strategy is to first find a local maximum in the orthogonal
subspace, then run the local optimization algorithm again (in the original
$n$-dimensional space) to ``refine'' the local maximum we have
found. The intuition is that since we are already close to a particular local maxima, the local search algorithm cannot jump to some other local maxima (since this would entail going through a valley).

\subsection{Finding one local maximum}
\label{subsec:singleopt}

Throughout this section, we will assume that we are given oracle access to a function
$f(u)$ and its gradient and Hessian.  The  procedure is also
given a starting point $u_s$, a search range $\beta$, and a step size $\delta$.
For simplicity in notation we define the following projection operator.
\begin{definition}
$\Proj_{\perp u}(v) = v - (u^T v) u$, $\Proj_{\perp u}(M) = M - (u^T M u) uu^T$. 
\end{definition}

The basic step the algorithm is a modification of Newton's method to find a local improvement that makes progress so long as
the current point $u$ is far from a local maxima. Notice that if we add a small vector to $u$, we do not necessarily preserve
the norm of $u$. In order to
have control over how the norm of $u$ changes, during local optimization step the algorithm
projects the gradient $\nabla f$ and Hessian $\mathcal{H}(f)$ to the
space perpendicular to $u$.  There is also an additional correction term $- \partial / \partial_u f(u)\cdot \|\xi\|^2/2$.
This correction term is necessary because the new vector we obtain is $(u+\xi)/\norm{(u+\xi)}_2$ which is close to $u-\|\xi\|_2^2/2 \cdot u + \xi +O(\beta^3)$.  Step 2 of the algorithm is
just maximizing a quadratic function and can be solved exactly using
Lagrangian Multiplier method. To increase efficiency it is also acceptable to perform an approximate
maximization step by taking $\xi$ to be either aligned with the gradient
$\Proj_{\perp u}\nabla f(u)$ or the largest eigenvector of
$\Proj_{\perp u}(\mathcal{H}(f(u)))$. 

\begin{fragment*}[t]
\caption{
\label{alg:localopt}{\sc LocalOPT}, \textbf{Input:}$f(u)$, $u_s$, $\beta$, $\delta$ \textbf{Output:} vector $v$\vspace*{0.01in}
}

\begin{enumerate} \itemsep 0pt
\small 
\item Set $u \leftarrow u_s$.
\item Maximize (via Lagrangian methods) 
$ \Proj_{\perp u}(\nabla f(u))^T \xi + \frac{1}{2} \xi^T \Proj_{\perp u}(\mathcal{H}(f(u))) \xi - \frac{1}{2} \left(\frac{\partial}{\partial_u} f(u)\right)\cdot \norm{\xi}_2^2 $
\\ Subject to  $ \norm{\xi}_2  \le \beta' $ and $u^T\xi = 0$

\item Let $\xi$ be the solution, $\tilde{u} = \frac{u+\xi}{\norm{u+\xi}}$
\item If $f(\tilde{u}) \ge f(u) + \delta/2$, set $u \leftarrow \tilde{u}$ and Repeat Step $2$
\item Else return $u$
\end{enumerate} 

\end{fragment*}

 The algorithm is guaranteed to succeed in polynomial time when the function is {\em Locally Improvable} and {\em Locally Approximable}:


\begin{definition}[($\gamma,\beta, \delta$)-Locally Improvable]
A function $f(u): \R^n\to \R$ is $(\gamma, \beta, \delta)$-Locally Improvable, if for any $u$ that is at least $\gamma$ far from any local maxima, there is a $u'$ such that $\norm{u'-u}_2 \le \beta$ and $f(u') \ge f(u) + \delta$.
\end{definition}

\begin{definition}[($\beta, \delta$)-Locally Approximable]
A function $f(u)$ is locally approximable, if its third order
derivatives exist and for any $u$ and any direction $v$, the
third order derivative of $f$ at point $u$ in the direction of $v$ is
bounded by $0.01\delta/\beta^3$. 
\end{definition}

The analysis of the running time of the procedure  comes from local
Taylor expansion. When a function is Locally Approximable it is well
approximated by the gradient and Hessian within a $\beta$
neighborhood. The following theorem from \cite{FJK} showed that the
two properties above are
 enough to guarantee the success of a local search algorithm even when the function is only approximated.

\begin{theorem}[\cite{FJK}]\label{thm:onemaximum} If $|f(u) - f^*(u)| \le \delta/8$, the function $f^*(u)$ is $(\gamma, \beta, \delta)$-Locally Improvable, $f(u)$ is $(\beta,\delta)$ Locally Approximable, then Algorithm~\ref{alg:localopt} will find a vector $v$ that is $\gamma$ close to some local maximum. The running time is at most $O((n^2 + T) \max f^* / \delta)$ where $T$ is the time to evaluate the function $f$ and its gradient and Hessian. 
\end{theorem}

\subsection{Finding all local maxima}

\label{subsec:allopt}

Now we consider how to find {\em all} local maxima of a given
function $f^*(u)$. The crucial condition that we need is that {\em all local
maxima are orthogonal} (which is indeed true in our problem, and is morally true when using local
search more generally in ICA). Note that this condition implies that there are at most $n$ local maxima.\footnotemark[1]
In fact we will assume that there are exactly $n$ local maxima.
If we are given an exact oracle for $f^*$ and can compute {\em exact} local maxima
then we can find all local maxima easily: find one local maximum,
project the function into the orthogonal subspace, and continue to
find more local maxima. 

\footnotetext[1]{Technically, there are $2n$ local maxima since for each direction $u$ that is a local maxima, so too is $-u$ but this is an unimportant detail for our purposes.}

\begin{definition}
The projection of a function $f$ to a linear subspace $S$ is a function on that subspace with value equal to $f$. More explicitly, if $\{v_1, v_2, ..., v_d\}$ is an orthonormal basis of $S$, the projection of $f$ to $S$ is a function $g: \R^d\rightarrow \R$ such that $g(w) = f(\sum_{i=1}^d w_i v_i)$. 
\end{definition}


\begin{fragment*}[t]
\caption{
\label{alg:allopt}{\sc AllOPT}, \textbf{Input:}$f(u)$, $\beta$, $\delta$, $\beta'$, $\delta'$ \textbf{Output:} $v_1, v_2, ..., v_n$, $\forall i$ $\norm{v_i - v^*_i} \le \gamma$.
}

\begin{enumerate} \itemsep 0pt
\small 
\item Let $v_1 = \mbox{\sc LocalOPT($f, e_1, \beta, \delta$)}$
\item FOR $i = 2$ TO $n$ DO
\item $\quad$ Let $g_i$ be the projection of $f$ to the orthogonal subspace of $v_1, v_2, ..., v_{i-1}$.
\item $\quad$ Let $u' =  \mbox{\sc LocalOPT($g, e_1, \beta', \delta'$)}$.
\item $\quad$ Let $v_i = \mbox{\sc LocalOPT($f, u', \beta,\delta$)}$.
\item END FOR
\item Return $v_1, v_2, ..., v_n$
\end{enumerate} 
\end{fragment*}

The following theorem gives sufficient conditions under which the
above algorithm finds all local maxima, making precise the intuition
given at the beginning of this section.
\begin{theorem}
\label{thm:allopt}
Suppose the function $f^*(u):\R^n\to \R$ satisfies the following properties:
\begin{enumerate} \itemsep -2pt
\item Orthogonal Local Maxima: The function has $n$ local maxima $v^*_i,$ and they are orthogonal to each other.
\item Locally Improvable: $f^*$ is $(\gamma, \beta, \delta)$ Locally Improvable.
\item Improvable Projection: The projection of the function to any subspace spanned by a subset of local maxima is $(\gamma', \beta', \delta')$ Locally Improvable. The step size $\delta' \ge 10\delta$. 
\item Lipschitz: If two points $\norm{u-u'}_2 \le 3\sqrt{n}\gamma$, then the function value $|f^*(u)-f^*(u')| \le \delta'/20$.
\item Attraction Radius: Let $Rad \ge 3\sqrt{n}\gamma+\gamma'$, for any local maximum $v^*_i$, let $T$ be $\min f^*(u)$ for $\norm{u-v^*_i}_2 \le Rad$, then there exist a set $U$ containing $\norm{u-v^*_i}_2 \le 3\sqrt{n}\gamma+\gamma'$ and does not contain any other local maxima, such that for every $u$ that is not in $U$ but is $\beta$ close to $U$, $f^*(u) < T$.
\end{enumerate}

If we are given function $f$ such that $|f(u) - f^*(u)|\le \delta/8 $ and $f$ is both $(\beta,\delta)$ and $(\beta',\delta')$  Locally Approximable, then Algorithm~\ref{alg:allopt} can find all local maxima of $f^*$ within distance $\gamma$.
\end{theorem}

To prove this theorem, we first notice the projection of the function $f$ in Step 3 of the algorithm should be close to the projection of $f^*$ to the remaining local maxima. This is implied by Lipschitz condition and is formally shown in the following two lemmas. First we prove a ``coupling" between the orthogonal complement of two close subspaces:

\begin{lemma}
\label{lem:projectspace}
Given $v_1$, $v_2$, ..., $v_k$, each $\gamma$-close respectively to local maxima $v^*_1, v^*_2, ..., v^*_k$ (this is without loss of generality because we can permute the index of local maxima), then there is an orthonormal basis $v_{k+1}, v_{k+2},...,v_{n}$ for the orthogonal space of $span\{v_1,v_2,...,v_k\}$ such that for any unit vector $w\in \R^{n-k}$, $\sum_{i=1}^{n-k} w_k v_{k+i}$ is $3\sqrt{n}\gamma$ close to $\sum_{i=1}^{n-k} w_k v^*_{k+i}$. 
\end{lemma}

\begin{proof}
Let $S_1$ be $span\{v_1, v_2, ..., v_k\}$, $S_2$ be $span\{v^*_1, v^*_2, ..., v^*_k\}$ and $S_1^\perp$, $S_2^\perp$ be their orthogonal subspaces respectively. We first prove that for any unit vector $v \in S_1^\perp$, there is another unit vector $v' \in S_2^\perp$ so that $v^T v' \ge 1-4n\gamma^2 $. In fact, we can take $v'$ to be the unit vector along the projection of $v$ in $S_2^\perp$. To bound the length of the projection, we instead bound the length of projection to $S_2$. Since we know $ v_i^T v' = 0$ for $i\le k$ and $\norm{v_i - v^*_i} \le \gamma$, it must be that $( v^*_i)^T v' \le 2\gamma$ when $\gamma<0.01$. So the projection of $v'$ in $S_2$ has length at most $2\sqrt{n}\gamma$ and hence the projection of $v'$ in $S_2^\perp$ has length at least $1-4n\gamma^2$. 

Next, we prove that there is a pair of orthornormal basis $\{\tilde{v}_{k+1}, \tilde{v}_{k+2}, ..., \tilde{v}_n\}$ and $\{\tilde{v^*}_{k+1}, \tilde{v^*}_{k+2}, ..., \tilde{v^*}_n\}$ for $S_1^\perp$ and $S_2^\perp$ such that $\sum_{i=1}^{n-k} w_k \tilde{v}_{k+i}$ is close to $\sum_{i=1}^{n-k} w_k \tilde{v^*}_{k+i}$. Once we have such a pair, we can simultaneously rotate the two basis so that the latter becomes $v^*_{k+1}, ..., v^*_n$. 

To get this set of basis we consider the projection operator to $S_2^\perp$ for vectors in $S_1^\perp$. The squared length of the projection is a quadratic form over the vectors in $S_1^\perp$. So there is a symmetric PSD matrix $M$ such that $\norm{\Proj_{S_2^\perp}(v)}_2^2 = v^T M v$ for $v\in S_1^\perp$. Let $\{\tilde{v}_{k+1}, \tilde{v}_{k+2}, ..., \tilde{v}_n\}$ be the eigenvectors of this matrix $M$. As we showed the eigenvalues must be at least $1-8n\gamma^2$. The basis for $S_2^\perp$ will just be unit vectors along directions of projections of $\tilde{v}_i$ to $S_2^\perp$. They must also be orthogonal because the projection operator is linear and $$\norm{\Proj_{S_2^\perp}(\sum_{i=1}^{n-k} w_i \tilde{v}_{k+i})}_2^2 =  \norm{\sum_{i=1}^{n-k} w_i \Proj_{S_2^\perp}(\tilde{v}_{k+i})}_2^2 =  \sum_{i=1}^{n-k} \lambda_i w_i^2$$ The second equality cannot hold if these vectors are not orthogonal.
And for any $w$, $$\left(\sum_{i=1}^{n-k} w_k \tilde{v}_{k+i}\right)^T \left(\sum_{i=1}^{n-k} w_k \tilde{v^*}_{k+i}\right) = \sum_{i=1}^{n-k} w_k^2 (\tilde{v}_{k+i})^T \tilde{v^*}_{k+i} \ge 1-8n\gamma^2$$ So we conclude that the distance between these two vectors is at most $3\sqrt{n}\gamma$. 
\end{proof}

Using this lemma we see that the projected function is close to the projection of $f^*$ in the span of the rest of local maxima:

\begin{lemma}
\label{lem:projfunction}
Let $g^*$ be the projection of $f^*$ into the space spanned by the rest of local maxima, then $|g^*(w) - g(w)| \le \delta/8 + \delta'/20 \le \delta'/8$. 
\end{lemma}

\begin{proof}
The proof is straight forward because $|g^*(w) - g(w)| \le |f^*(u)-f(u)| + |f^*(u)-f^*(u')|$ for some $\norm{u-u'}_2\le 3\sqrt{n}\gamma$, we know the first one is at most $\delta/8$ and the second one is at most $\delta'/20$ by Lipschitz Condition.
\end{proof}

Now we are ready to prove the main theorem.

\begin{proof} [Theorem~\ref{thm:allopt}]
By Theorem~\ref{thm:onemaximum} the first column is indeed $\gamma$ close to a local maximum. We then prove by induction that if $v_1$, $v_2$, ..., $v_k$ are $\gamma$ close to different local maxima, then $v_{k+1}$ must be close to a new local maximum.

By Lemma~\ref{lem:projfunction} we know $g_{k+1}$ is $(\gamma', \beta', \delta')$ Locally Improvable, and because it is a projection of $f$ its derivatives are also bounded so it is $(\beta', \delta')$ Locally Approximable. By Theorem~\ref{thm:onemaximum} $u'$ must be $\gamma'$ close to local maximum for the projected function. Then since the projected space is close to the space spanned by the rest of local maxima, $u'$ is in fact $\gamma'+3\sqrt{n}\gamma$ close to $v^*_{k+1}$ (here again we are reindexing the local maxima wlog.).

Now we use the Attraction Radius property, since $u$ is currently in $U$, $f^*(u) \ge T$, and each step we go to a point $u'$ such that $\norm{u'-u}\le \beta$ and $f^*(u') > f^*(u) \ge T$. The local search in Algorithm~\ref{alg:localopt} can never go outside $U$, therefore it must find the local maximum $v^*_{k+1}$. 
\end{proof}


\section{Local search on the fourth order cumulant}
\label{sec:cumulant}

Next, we prove that the fourth order cumulant $P^*(u)$ satisfies the properties above. Then the algorithm given in the previous section will find all of the local maxima, which is the missing step in our main goal: learning a noisy linear transformation $Ax + \eta$ with unknown Gaussian noise. We first use a theorem from \cite{FJK} to show that properties for finding one local maxima is satisfied.

Also, for notational convenience we set $d_i = 2 D_A(u_0)_{i,i}^{-2}$ and let $d_{min}$ and $d_{max}$ denote the minimum and maximum values (bounds on these and their ratio follow from Claim \ref{claim:obv}). Using this notation $P^*(u) = \sum_{i=1}^n d_i (u^TR^*_i)^4$.

\begin{theorem}[\cite{FJK}]\label{thm:interface}When $\beta < d_{min}/10d_{max}n^2$, the function $P^*(u)$ is $(3\sqrt{n}\beta, \beta, P^*(u) \beta^2/100)$ Locally Improvable and $(\beta, d_{min} \beta^2/100n)$ Locally Approximable. Moreover, the local maxima of the function is exactly $\{\pm R^*_i\}$.
\end{theorem}
\begin{proof}
The proof appears in \cite{FJK}. Here for completeness we show the proof using our notations.

First we establish that $P^*(u)$ is Locally Improvable.Observe that this desirada is invariant under rotation, so we need only prove the theorem for $P^*(v) = \sum_{i=1}^n d_i v_i^4$. The gradient of the function is $\nabla P^*(v) = 4 (d_1v_1^3, d_2v_2^3, ..., d_nv_n^3)$. The inner product of $\nabla P^*(v) $ and $v$ is exactly $4\sum_{i=1}^n d_iv_i^4 = 4P^*(v)$. 
Therefore the projected gradient $\phi = \Proj_{\perp v}\nabla P^*(v) $ has coordinate $\phi_i = 4 v_i(d_iv_i^2-P^*(v))$. Furthermore, the Hessian $H = \calH(P^*(v))$ is a diagonal matrix whose $(i, i)^{th}$ entry is $12 d_i v_i^2$. 

Consider the case in which $\norm{\phi} \ge P^*(v) \beta/4$. We can obtain an improvement to $P^*(v) \beta^2/100$ because we can take $\xi$ in the direction of $\phi$ and with $\norm{\xi}_2 = \beta/20$. The contribution of the Hessian term is nonnegative and the third term $-2P^*(u) \norm{\xi}_2^2$ is small in comparison.


Hence, we can assume $\norm{\phi} \le P^*(v) \beta/4$. Now let us write out the expression of $\norm{\phi}^2$
$$
\sum_{i=1}^n v_i^2 (d_iv_i^2 - P^*(v))^2 \le \beta^2 (P^*(v))^2/16.
$$
\noindent In particular every term $v_i^2 (d_i v_i^2 - P^*(v))^2$ must be at most $\beta^2 (P^*(v))^2/16.$. Thus for any $i$, either $v_i^2 \le \beta^2$ or $(d_i v_i^2 - P^*(v))^2 \le (P^*(v))^2/16$. 

If there are at least 2 coordinates $k$ and $l$ such that $(d_i v_i^2 - P^*(v))^2 \le (P^*(v))^2/16$, then we know for these two coordinates $v_i^2 \in [0.75 P^*(v) /d_i, 1.25 P^*(v)/d_i]$. We choose the vector $\xi$ so that $\xi_k = \tau v_l$ and $\xi_l = -\tau v_k$. Wlog assume $\xi \cdot \phi \ge 0$ otherwise we use $-\xi$. Take $\tau$ so that $\tau^2 (v_l^2+v_k^2) = \beta^2$. Clearly $\norm{\xi} = \beta$ and $\xi \cdot v = 0$ so $\xi$ is a valid solution. Also $\tau^2$ is lower bounded by $\beta^2/(v_l^2+v_k^2) \ge \frac{4}{5} \frac{\beta^2}{P^*(u) (1/d_l+1/d_k)}$.

Consider the function we are optimizing:

\begin{align*}
 \phi \cdot \xi + 1/2 \xi^T \calH \xi - 2P^*(u)\norm{\xi}_2  &\ge 1/2 \xi^T H \xi - 2P^*(u)\beta^2  = 6\tau^2 v_k^2 v_l^2 (d_k+d_l) - 2P^*(u)\beta^2 \\
&  \ge \frac{27}{8} \tau^2 P^*(u)^2  \frac{d_k+d_l}{d_kd_l}  - 2P^*(u)\beta^2    \ge \frac{7}{10}  P^*(u)\beta^2.
\end{align*}

In the remaining case, all of the coordinates except for at most one satisfy $v_i^2 \le \beta^2$. Since we assumed $\beta^2 < \frac{1}{n}$, there must be one of the coordinate $v_k$ that is large, and it is at least $1-n\beta^2$. Thus the distance of this vector to the local maxima $e_k$ is at most $3\sqrt{n}\beta$.
\end{proof}

We then observe that given enough samples, the empirical mean $\widehat{P}'(u)$ is close to $P^*(u)$. For concentration we require every degree four term $z_iz_jz_kz_l$ has variance at most $Z$.
\begin{claim}
$Z = O(d_{min}^2\lambda_{min}(A)^8 \|\Sigma\|_2^4+ d_{min}^2)$.
\end{claim}

\begin{proof}
We will start by bounding $\E[(z_iz_jz_kz_l)^2] \le \E[(z_i^8+z_j^8+z_k^8+z_l^8)]$. Furthermore $\E[z_i^8] \le O(\E[(B^{-1}A x)_i^8+ (B^{-1}\eta)_i^8])$. 
Next we bound $\E[(B^{-1}\eta)_i^8]$, which is just the eighth moment of a Gaussian with variance at most $\| B^{-1}\Sigma B^{-T}\|_2 \le \|B^{-1}\|_2^2 \|\Sigma\|_2 \le d_{min}^{1/2} \lambda_{min}(A)^{-2} \|\Sigma\|_2$. Hence we can bound this term by$O(\| B^{-1}\Sigma B^{-T}\|_2^4) = O(d_{min}^2\lambda_{min}(A)^8 \|\Sigma\|_2^4)$. Finally the remaining term $\E[(B^{-1}A x)_i^8]$ can be bounded by $O(d_{min}^2)$ because the variance of this random variable is only larger if we instead replace $x$ by an $n$-dimensional standard Gaussian.
\end{proof}

\begin{lemma} \label{lem:functionvalue}
Given $2N$ samples $y_1, y_2, ..., y_N, y_1', y_2', ..., y_N'$, suppose columns of $R' = B^{-1}AD_A(u_0)^{1/2}$ are $\epsilon$ close to the corresponding columns of $R^*$, with high probability the function $\widehat{P}'(u)$ is $O(d_{max}n^{1/2}\epsilon + n^2 (N/Z\log n)^{-1/2})$ close to the true function $P^*(u)$.
\end{lemma}

\begin{proof}
$\widehat{P}'(u)$ is the empirical mean of $F(u,y,y') = - (u^T B^{-1} y)^4 + 3 (u^T B^{-1} y)^2(u^TB^{-1}y')^2$.
In Section~\ref{sec:whitening} we proved that $P'(u) = \E_{y,y'} F(u,y,y') = \sum_{i=1}^n 2 D^{-1/2}_{i,i} (u^T R_i)^4 = \sum_{i=1}^n \lambda_i (u^T R_i)^4$. First, we demonstrate that $P'(u)$ is close to $P^*(u)$, and then using concentration bounds we show that $\widehat{P}'(u)$ is close to $P'(u)$ (with high probability) over all $u$.

The first part is a simple application of Cauchy-Schwartz:

\begin{align*}
|P'(u) - P^*(u)| & = \sum_{i=1}^n d_i \left[(u^TR'_i) - (u^TR^*_i)\right]\cdot \left[(u^TR'_i + u^TR^*_i) ((u^TR'_i)^2 +(u^TR^*_i)^2)\right]\\ & \le d_{max} \sqrt{ \sum_{i=1}^n (u^T(R'_i-R^*_i))^2} \cdot (3 \norm{u^TR'+u^TR^*}_2) \le 6 d_{max} n^{1/2} \epsilon. \label{eqn:lipschitz} 
\end{align*}
\noindent The first inequality uses the fact that $((u^TR'_i)^2 +(u^TR^*_i)^2) \le 3$, the second inequality uses the fact that when $\epsilon$ is small enough, $\norm{u^T R'}_2 \le 2$.

Next we prove that the empirical mean $\widehat{P}'(u)$ is close to $P'(u)$. The key point here is we need to prove this for all points $u$ since a priori we have no control over which directions local search will choose to explore. We accomplish this by considering $\widehat{P}'(u)$ as a degree-4 polynomial over $u$ and prove that the coefficient of each monomial in $\widehat{P}'(u)$ is close to the corresponding coefficient in $P'(u)$. This is easy: the expectation of each coefficient of $F(u,y,y')$ is equal to the correct coefficient, and the variance is bounded by $O(Z)$. The coefficients are also sub-Gaussian so by Bernstein's inequality the probability that any coefficient of $\widehat{P}'(u)$ deviates by more than $\epsilon'$ (from its expectation) is at most $e^{- \Omega(\epsilon'^2 N/Z)}$. Hence when $N \ge O(Z\log n/\epsilon'^2)$ with high probability all the coefficients of $\widehat{P}'(u)$ and $P'(u)$ are $\epsilon'$ close. So for any $u$:
$$
|P'(u) - \widehat{P}'(u)|  \le \epsilon' (\sum_{i=1}^n |u_i|)^4 \le \epsilon' n^2.
$$
\noindent Therefore $\widehat{P}'(u)$ and $P^*(u)$ are $O(d_{max}n^{1/2}\epsilon + n^2 (N/Z\log n)^{-1/2})$ close. 
\end{proof}

This proof can also be used to show that the derivatives of the function $\widehat{P}'(u)$ is concentrated to the derivatives of the true function $P^*(u)$ because the derivatives are only related to coefficients, therefore $\widehat{P}'(u)$ is also $(\beta, d_{min}\beta^2/100n)$ Locally Approximable.

The other properties required by Theorem~\ref{thm:allopt} are also satisfied:

\begin{lemma}
For any $\|u-u'\|_2 \le r$, $|P^*(u)-P^*(u')| \le 5d_{max}n^{1/2}r$. All local maxima of $P^*$ has attraction radius $Rad\ge d_{min}/100d_{max}$. 
\end{lemma}

\begin{proof}
 The Lipschitz condition follows from the same Cauchy-Schwartz as appeared above. When two points $u$ and $u'$ are of distance $r$, $|P^*(u) - P^*(u')| \le 5d_{max} n^{1/2} r$. Finally for the Attraction Radius, we know when $3\sqrt{n}\gamma+\gamma' \le  d_{min}/100d_{max}$, we can just take the set $U$ to be $u^T R^*_i \ge1-d_{min}/50d_{max}$. For all $u$ such that $u^TR^*_i \in [1-d_{min}/25d_{max}, 1-d_{min}/50d_{max}]$ (which contains the $\beta$ neighborhood of $U$), we know the value of $P^*(u) \le T$.
\end{proof}

Applying Theorem~\ref{thm:allopt} we obtain the following Lemma (the parameters are chosen so that all properties required are satisfied):

\begin{lemma}
\label{lem:alllocalmax}
Let $\beta' = \Theta((d_{min}/d_{max})^2)$,  $\beta = \min\{\gamma n^{-1/2}, \Omega((d_{min}/d_{max})^{4} n^{-3.5})\}$, then the procedure {\sc Recover($f,\beta,$ $ d_{min}\beta^2/100n$ $, \beta', d_{min}\beta'^2/100n$)} finds vectors $v_1, v_2,..., v_n$, so that there is a permutation matrix $\Pi$ and $k_i\in \{\pm 1\}$ and for all $i$: $\norm{v_i-(R\Pi \mbox{Diag}(k_i))^*_i}_2 \le \gamma$.
\end{lemma}

After obtaining $\widehat{R} = \left[ v_1, v_2, ..., v_n\right]$ we can use Algorithm~\ref{alg:recover} to find $A$ and $\Sigma$:

\begin{fragment*}[t]
\caption{
\label{alg:recover}{\sc Recover}, \textbf{Input:}$B$, $\widehat{P}'(u)$, $\widehat{R}$, $\epsilon$ \textbf{Output:} $\widehat{A}$, $\widehat{\Sigma}$
}

\begin{enumerate} \itemsep 0pt
\small 
\item Let $\widehat{D}_A(u)$ be a diagonal matrix whose $i^{th}$ entry is $\frac{1}{2}\left(\widehat{P}'(\widehat{R}_i)\right)^{-1/2}$.
\item Let $\widehat{A} = B\widehat{R}\widehat{D}_A(u)^{-1/2}$.
\item Estimate $C = \E[yy^T]$ by taking $O((\norm{A}_2+\norm{\Sigma}_2)^4 n^2\epsilon^{-2})$ samples and let $\widehat{C} = \frac{1}{N}\sum_{i=1}^N y_iy_i^T$.
\item Let $\widehat{\Sigma} = \widehat{C} - \widehat{A}\widehat{A}^T$
\item Return $\widehat{A}, \widehat{\Sigma}$
\end{enumerate} 

\end{fragment*}

\begin{theorem}
\label{thm:recover}
Given a matrix $\widehat{R}$ such that there is permutation matrix $\Pi$ and $k_i\in \{\pm 1\}$ with $\|\widehat{R}_i - k_i (R^*\Pi)_i\|_2 \le \gamma$ for all $i$, Algorithm~\ref{alg:recover} returns matrix $\widehat{A}$ such that $\|\widehat{A} - A\Pi\mbox{Diag}(k_i)\|_F \le O(\gamma  \norm{A}_2^2  n^{3/2}/\lambda_{min}(A))$. If $\gamma\le O(\epsilon/\norm{A}_2^2n^{3/2} \lambda_{min}(A)) \times \min\{1/\norm{A}_2, 1\}$, we also have $\|\widehat{\Sigma}-\Sigma\|_F \le \epsilon$.
\end{theorem}

Recall that the diagonal matrix $D_A(u)$ is unknown (since it depends on $A$), but if we are given $R^*$ (or an approximation) and since $P^*(u) = \sum_{i=1}^n d_i (u^TR^*_i)^4$, we can recover the matrix $D_A(u)$ approximately from computing $P^*(R^*_i)$. Then given $D_A(u)$, we can recover $A$ and $\Sigma$ and this completes the analysis of our algorithm. 

\begin{proof}
By Lemma~\ref{lemma:rot} we know the columns of $R'$ is close the the columns of $R$ (the parameters will be set so that the error is much smaller than $\gamma$), thus $\|\widehat{R}_i - k_i (R'\Pi)_i\|_2 \le \gamma$. Applying Lemma~\ref{lem:functionvalue} we obtain: $|\widehat{P}'(\widehat{R}_i) - P^*(\widehat{R}_i)| \ll \gamma$. Furthermore, when $\|\widehat{R}_i - k_i R^*_{\Pi^{-1}(i)}\|_2 \le \gamma$ we know that $P^*(\widehat{R}_i) / d_{\Pi^{-1}(i)} \in [1-3\gamma,1+3\gamma]$ (here we are abusing notation and use the permutation matrix as a permutation). Hence $\widehat{D}_A(u)_{i,i} /  \left(D_A(u)\right)_{\Pi^{-1}(i), \Pi^{-1}(i)} \in [1-3\gamma,1+3\gamma]$. We have: 

$$\widehat{A}_i = B\widehat{R}_i \widehat{D}_A(u)_{i,i}^{-1/2} \mbox{ and } \left(A\Pi\mbox{Diag}(k_i)\right)_i = BR'_{\Pi^{-1}(i)} \left(D_A(u)\right)_{\Pi^{-1}(i), \Pi^{-1}(i)}^{-1/2}$$

\noindent and their difference is at most $O(\gamma \norm{B}_2  \left(D_A(u)\right)_{\Pi^{-1}(i), \Pi^{-1}(i)}^{-1/2})$. Hence we can bound the total error by $O(\gamma \norm{B}_2 \norm{D_A(u)^{-1/2}}_F)$. We also know $\norm{B}_2 \le \norm{A}_2 \|D_A(u)^{1/2}\|_2$ because $BB^T \approx AD_A(u)A^T$, so this can be bounded by $O(\gamma \norm{A}_2 \|D_A(u)\|_2^{1/2} \|D_A(u)^{-1/2}\|_F)$. Applying Claim~\ref{claim:obv}, we conclude that (with high probability) the ratio of the largest to smallest diagonal entry of $D_A(u)$ is at most $n^2\norm{A}_2^2/\lambda_{min}(A)^2$. So we can bound the error by $O(\gamma \norm{A}_2^2 n^{3/2}/\lambda_{min}(A))$. 

Consider the error for $\Sigma$: Using concentration bounds similar but much simpler than those used in Lemma~\ref{lem:functionvalue}, we obtain that $\|\widehat{C} - C\|_F \le 1/2 \epsilon$, so $\|\widehat{\Sigma} - \Sigma\|_F \le \|\widehat{C} - C\|_F - \|\widehat{A}\widehat{A}^T - AA^T\|_F \le \epsilon/2 + 2 \norm{A}_2 \|A\Pi\mbox{Diag}(k_i) - \widehat{A}\|_F + \|A\Pi\mbox{Diag}(k_i) - \widehat{A}\|_F^2 \le \epsilon$.
\end{proof}

\subsubsection*{Conclusions} 

ICA is a vast field with many successful
techniques. Most rely on heuristic nonlinear optimization. An exciting
question is: can we give a rigorous analysis of those techniques as
well, just as we did for local search on cumulants?  A rigorous
analysis of deep learning ---say, an algorithm that provably learns
the parameters of an RBM---is another problem that is wide open, and a plausible special case
involves subtle variations on the problem we considered here.




\end{document}